\documentclass[letterpaper, 10 pt, conference]{ieeeconf}  % Comment this line out if you need a4paper
\usepackage[T1]{fontenc}

\usepackage{amsthm} % for definition, theorm and etc
\usepackage{amsmath,bm}
\usepackage{amsfonts} 
\usepackage{color}
\usepackage{makecell}
\usepackage{centernot}
\usepackage{graphicx} 

\usepackage{subfig}
\usepackage{multirow}
\usepackage{booktabs}
\usepackage{mathtools}
\usepackage{cite}
\usepackage{url}
\usepackage[table]{xcolor}

\newtheorem{lemma}{Lemma}
\newtheorem{definition}{Definition}
\newtheorem{theorem}{Theorem}

\newtheorem{problem}{Problem}
\newtheorem{remark}{Remark}

\newcommand{\xinyi}[1]{\textcolor{black}{#1}}

\usepackage{hyperref}
\IEEEoverridecommandlockouts                              
\DeclareCaptionFont{mysize}{\fontsize{8}{9.6}\selectfont}
\title{\LARGE \bf
Safe Navigation in Uncertain Crowded Environments Using Risk Adaptive CVaR Barrier Functions
}

\author{Xinyi Wang, Taekyung Kim, Bardh Hoxha, Georgios Fainekos and Dimitra Panagou
\thanks{Xinyi Wang and Taekyung Kim are with the Department of Robotics, University of Michigan, Ann Arbor, MI, 48109, USA {\tt\footnotesize \{xinywa, taekyung\}@umich.edu}}
\thanks{Bardh Hoxha and Georgios Fainekos are with the Toyota North America Research \& Development, 1555 Woodridge Ave,
Ann Arbor, MI 48105, USA {\tt\footnotesize \{firstname.lastname\}@toyota.com}}
\thanks{Dimitra Panagou is with the Department of Robotics and the Department of Aerospace Engineering at University of Michigan, Ann Arbor, MI 48109, USA {\tt\footnotesize dpanagou@umich.edu}}
}

\begin{document}

\maketitle
\thispagestyle{empty}
\pagestyle{empty}

%%%%%%%%%%%%%%%%%%%%%%%%%%%%%%%%%%%%%%%%%%%%%%%%%%%%%%%%%%%%%%%%%%%%%%%%%%%%%%%%
\begin{abstract}
Robot navigation in dynamic, crowded environments poses a significant challenge due to the inherent uncertainties in the obstacle model.  In this work, we propose a risk-adaptive approach based on the Conditional Value-at-Risk Barrier Function (CVaR-BF), where the risk level is automatically adjusted to accept the minimum necessary risk, achieving a good performance in terms of safety and optimization feasibility under uncertainty. Additionally, we introduce a dynamic zone-based barrier function which characterizes the collision likelihood by evaluating the relative state between the robot and the obstacle. By integrating risk adaptation with this new function, our approach adaptively expands the safety margin, enabling the robot to proactively avoid obstacles in highly dynamic environments. Comparisons and ablation studies demonstrate that our method outperforms existing social navigation approaches, and validate the effectiveness of our proposed framework. \href{https://lawliet9666.github.io/cvarbf/}{[Paper Page]}\,\href{https://youtu.be/VHRnmXToLN8?si=lV8pyZSkYPf1puw_}{[Video]}\,\href{https://github.com/Lawliet9666/Adaptive-CVaR-Barrier-Function.git}{[Code]}.

\end{abstract}

\section{Introduction}
Safe navigation in crowded environments with dynamic obstacles remains a fundamental challenge in robotics due to obstacle uncertainty. 
The common approaches involve using risk metrics to quantify and enforce safety constraints under this uncertainty, such as conventional stochastic control methods \cite{lekeufack2024conformal,fushimi2025safety,black2024risk} and robust
control methods \cite{malone2017hybrid, safaoui2024distributionally}. While robust control methods prioritize worst-case scenarios, often leading to overly conservative behaviors, stochastic methods optimize for expected performance, which may result in unsafe decisions in high-risk situations. 

% Recently, \cite{ahmadi2021risk,safaoui2024distributionally,kishida2024risk} have explored combining Conditional Value-at-Risk (CVaR) with Control Barrier Functions (CBFs) to achieve probabilistic safety, thereby providing a trade-off between conservatism and efficiency. 
The combination of Conditional Value-at-Risk (CVaR) and Control Barrier Functions (CBFs) has recently been explored to achieve probabilistic safety, offering a trade-off between conservatism and efficiency \cite{ahmadi2021risk,safaoui2024distributionally,kishida2024risk}.
In these frameworks,
CBFs \cite{garg2024advances, kim2025learning} act as safety filters that regulate control inputs, while CVaR \cite{alonso2013optimal} quantifies the expected risk under a given risk level.
However, existing methods based on the CVaR barrier function (CVaR-BF) primarily rely on a \emph{fixed} risk level that remains constant throughout the trajectory. 
This can limit their applicability in dynamic environments with crowded obstacles. 
% We highlight that keeping a fixed risk level results in an issue:
% We highlight an important issue with maintaining a fixed risk level: 
We highlight that such a fixed risk level is not flexible enough:
A low risk tolerance enhances safety but can render the optimization infeasible, whereas a high risk tolerance improves feasibility at the expense of safety.

To address this limitation, we propose a risk-adaptive navigation approach, a novel extension of the CVaR-BF framework that dynamically adjusts the risk level to maintain safety while ensuring trajectory feasibility. 
Instead of relying on a pre-specified risk parameter, our adaptive controller initializes the system with a conservative risk level (e.g., zero) and incrementally increases it until an optimization-feasible solution is achieved.
In highly dynamic scenarios, where obstacles move unpredictably and rapidly, the robot requires sufficient time and space to respond and adjust its risk level. 
However, overly conservative strategies can limit feasible solutions in crowded environments \cite{tayal2024control, roncero2024multi}, so an approach that maintains safety without excessively restricting the decision space is essential.
To this end, we introduce the concept of a ``Dynamic Zone", where the original safety distance is adaptively expanded based on the relative position and velocity between the robot and its surrounding obstacles.
The robot adjusts its trajectory before nearing obstacles, but only when needed to avoid unnecessary conservatism, while also extending the risk-adaptive range.\\
The main contributions of this paper are as follows:\\
\begin{figure}
    \centering
    \includegraphics[width=0.73\linewidth]{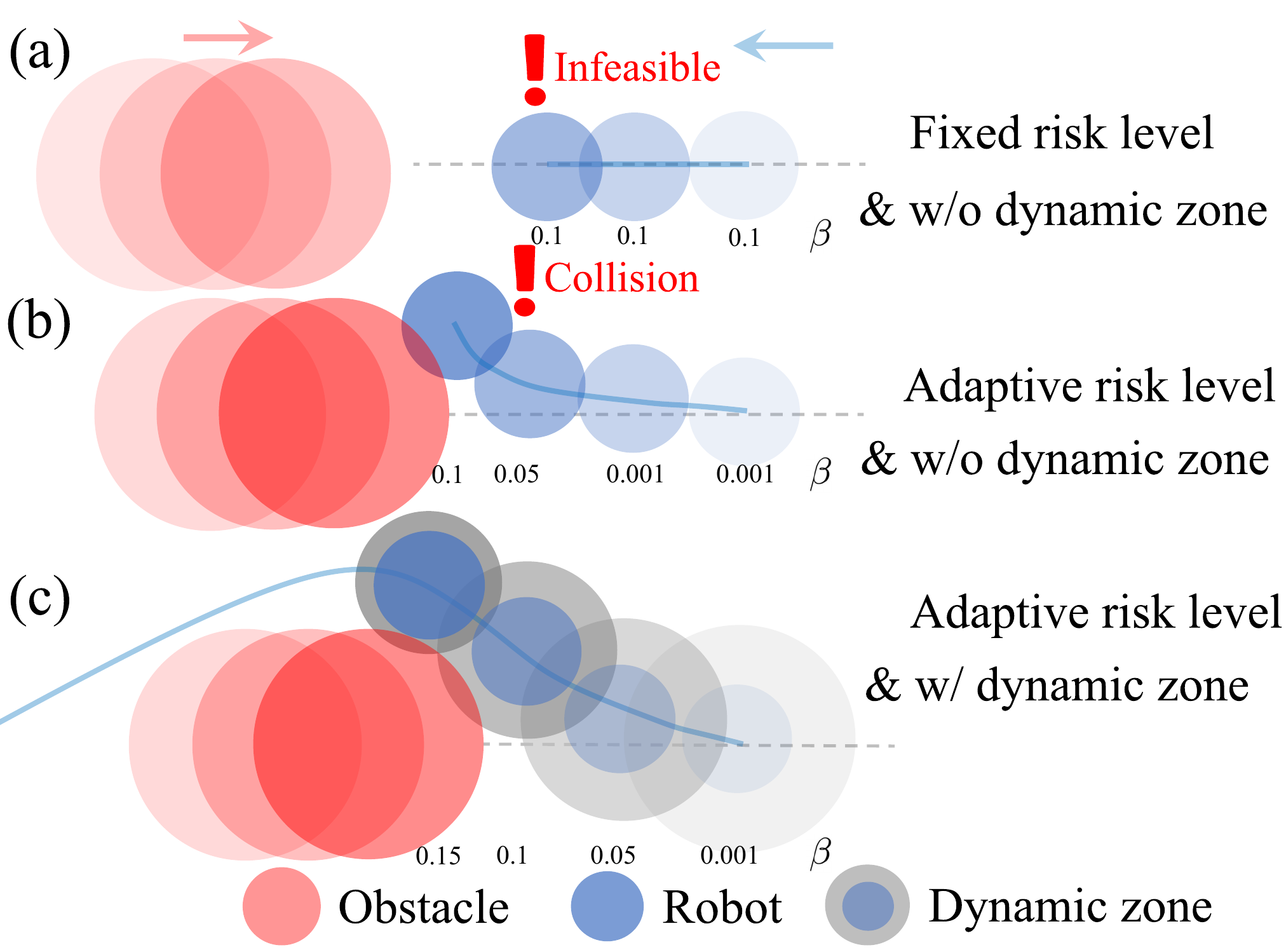}
    \caption{Comparison of fixed vs adaptive risk levels with and without a dynamic zone.}
    \label{fig:zone}
    \vspace{-12pt}
\end{figure}
(1) We introduce an adaptive strategy that adjusts the risk level to adopt the minimum necessary risk when navigating through obstacles, ensuring CVaR safety is guaranteed at least a pre-defined threshold while improving optimization feasibility.\\
(2) To increase the responsiveness to dynamic obstacles, we design a dynamic zone-based barrier function, which expands the available adjustment space for the risk level while maintaining the desired probabilistic safety guarantee.\\
(3) Empirical results demonstrate that our approach outperforms state-of-the-art methods in highly dynamic environments, achieving better success rates and more robustness under various of uncertain settings.

\section{Related Work}

\subsubsection{Risk-Aware Control under Uncertainties}
Risk-aware control under uncertainties can be approached using a variety of risk measures.
The approach in \cite{lekeufack2024conformal} achieves finite-time risk bounds by considering expected risk and penalizing collisions in the cost function. Similarly, \cite{fushimi2025safety, cosner2023robust} develop probabilistic safety bounds over a finite time horizon using a discrete-time CBF condition. 
However, these approaches prioritize average behavior and do not provide safety guarantee at each time step. 
Recently, distributionally robust optimization has been designed to enhance the safety by considering a set of possible distributions to optimize decisions for worst-case scenarios.
\cite{safaoui2024distributionally} reformulates the collision avoidance problem by computing safe half-spaces from obstacle sample trajectories via distributionally-robust optimization. 
\cite{malone2017hybrid} considers the  reachable set of the states of an obstacle under worst-case noise. Although these methods offer robust safety guarantees, they tend to be overly conservative, which can lead to infeasible solutions in dense environments.

\subsubsection{Dynamic Obstacle Avoidance}
% Recent work spans geometric methods, learning-based techniques, and optimization frameworks.
%One classic approach on dynamic obstacle avoidance is to employ geometric methods. 
Geometric techniques such as Velocity Obstacles (VO) \cite{van2008reciprocal} and Optimal Reciprocal Collision Avoidance (ORCA) \cite{alonso2013optimal} utilize a velocity-obstacle framework that considers both the robot's and obstacles' velocities for avoidance. However, these methods typically do not account for the robot's dynamic model. Another popular trend involves learning-based methods \cite{berducci2024learning}. For instance,
\cite{liu2023intention} prevents the robot from encroaching on the intended paths of other obstacles. Despite their promise, these methods often lack safety analysis and face generalization challenges when dealing with out-of-distribution data.

For safety-critical navigation, a rigorous theoretical framework and analysis are required. In \cite{samavi2024sicnav}, a coupled planning approach is proposed to model the obstacles' motion and solve a joint optimization problem with explicit safety constraints. However, coupled methods falter when the obstacle model is inaccurate \cite{han2024dr}, and accurately modeling obstacle motion is challenging.
% Control Barrier Functions (CBFs) 
CBFs \cite{garg2024advances, kim2025learning} have been developed as safety filters that mitigate unsafe control inputs. Some approaches partition the solution space into convex regions using separating hyperplanes \cite{liu2024safety, safaoui2024distributionally}. For example, \cite{safaoui2024distributionally} reformulates the collision avoidance problem by computing safe half-spaces based on dynamic obstacle sample trajectories.
Recent work has proposed a CBF approach combined with velocity obstacles \cite{tayal2024control, roncero2024multi, haraldsen2024safety}, where the velocity obstacle principle allows the vehicle to maintain forward motion without compromising safety. 
However, this method tends to trigger avoidance maneuvers even when obstacle are far away, thereby leading to an conservative behavior.
% However, this approach tends to exhibit conservative behavior as the robot aims to avoid obstacle while being still far away. 

\section{Preliminaries}
\subsection{Discrete-Time Control Barrier Functions}
Consider a robot whose motion is modeled by a discrete-time control system:
\begin{equation}
\label{eq:sys}
    \mathbf{x}_{k+1} = f(\mathbf{x}_k,\mathbf{u}_k),
\end{equation}
\xinyi{
where \(\mathbf{x}_k \in \mathcal{X} \subset \mathbb{R}^n\) is the state at time step \(k \in \mathbb{Z}^+\)}, and \(\mathbf{u}_k \in \mathcal{U} \subset \mathbb{R}^m\) is the control input, with \(\mathcal{U}\) being the set of admissible controls for system~\eqref{eq:sys}. The function \xinyi{\(f : \mathbb{R}^n \to \mathbb{R}^n\)} is assumed to be locally Lipschitz continuous with respect to  \(\mathbf{x}_k\) and \(\mathbf{u}_k\).

% where \(\mathbf{x}_k =[\mathbf{p}_k, \mathbf{v}_k ]^T \in \mathcal{X} \subset \mathbb{R}^n\) is the state at time step \(k \in \mathbb{Z}^+\), with $\mathbf p_k \in \mathbb{R}^d$ and $\mathbf v_k\in \mathbb{R}^d$ denoting the position and velocity of the robot w.r.t. a global frame (with \(n = 2d\) for a \(d\)-dimensional space), respectively, and \(\mathbf{u}_k \in \mathcal{U} \subset \mathbb{R}^m\) is the control input, with \(\mathcal{U}\) being the set of admissible controls for system~\eqref{eq:sys}. The function \(f : \mathbb{R}^n \to \mathbb{R}^n\) is assumed to be locally Lipschitz continuous with respect to  \(\mathbf{x}_k\) and \(\mathbf{u}_k\).

According to \cite{zeng2021safety}, the function $h$ is a discrete-time CBF for system~\eqref{eq:sys} if there exists an extended class \(\mathcal{K}_\infty\) function \(\alpha : \mathbb R \rightarrow \mathbb R\) such that 
there exists a control input \(\mathbf{u}_k\) satisfying
\begin{equation}
\Delta h(\mathbf{x}_{k}, \mathbf{u}_{k}) \coloneqq h(\mathbf{x}_{k+1}) - h(\mathbf{x}_k)\geq -\alpha\big(h(\mathbf{x}_k)\big).
\label{eq:dt_cbf}
\end{equation}
Following \cite{zeng2021safety}, we choose a linear function \(\alpha(r) = \gamma\, r\) with \(0 < \gamma \leq 1\), so that the condition~\eqref{eq:dt_cbf} becomes
\begin{equation}
h(\mathbf{x}_{k+1}) \geq (1-\gamma) h(\mathbf{x}_k),
\end{equation}
which guarantees that the value of \(h(\mathbf{x})\) decays at an exponential rate governed by \(1-\gamma\). Further, given a nominal control input \(\bar{\mathbf{u}}_k\), a CBF-based controller can be constructed to minimally modify \(\bar{\mathbf{u}}_k\) to guarantee safety:
\begin{equation}
\label{eq:qp}
\min_{\mathbf{u}_k\in \mathcal{U}}  \|\mathbf{u}_k - \bar{\mathbf{u}}_k\|^2, \enspace \text{s.t.} \enspace h(\mathbf{x}_{k+1}) \geq (1-\gamma) h(\mathbf{x}_k).
\end{equation}

\subsection{Probabilistic Constraints}

Due to the inherent uncertainty in the dynamic obstacle scenarios, the safe states of the robot cannot be computed deterministically. Therefore, a probabilistic formulation of safety is required. 
\xinyi{
The estimated obstacle state at time $k$, denoted $\hat{\mathbf{x}}^o_k \in \mathcal{O} \subset \mathbb{R}^n$, is obtained via sensor data and state estimation, and is treated as a deterministic value in our framework. However, the future state of the obstacle $\mathbf{x}^o_{k+1}$ is modeled as a random variable, reflecting the inherent uncertainty in the predicted motion of the obstacle.
Define the safe set as
\begin{equation}
\label{eq:safeset}
\mathcal{S} = \{\mathbf{x}_k \in \mathcal{X} : h(\mathbf{x}_k, \xinyi{\hat{\mathbf{x}}^o_k)} \geq 0\} .
\end{equation}
\(h: \mathcal{X} \times \mathcal{O} \to \mathbb{R}\) is the CBF that depends on both the system state \(\mathbf{x}_k\) and the estimated obstacle state \(\hat{\mathbf{x}}^o_k\) at the current time step~$k$. Therefore}, the CBF at the next time step $k+1$ considering the predicted state, which is denoted as 
\(
h_{k+1} \coloneqq h\big(\mathbf{x}_{k+1}, \mathbf{x}^o_{k+1}\big),
\) becomes a random variable due to the uncertainty in the obstacle state \(\mathbf{x}^o_{k+1}\) 
\footnote{For notational brevity, we omit the explicit argument of the CBF~$h$ when its dependency is clear from the context.}. 

To account for this uncertainty, we enforce safety to the system using the following probabilistic constraints:
\begin{equation}
\mathbb{P}\bigl( h_{k+1} \geq 0 \bigr) \geq 1 - \beta,
\end{equation}
% \end{definition}
where $\beta \in (0, 1)$ indicates the allowable probability of collisions.
According to \cite{sarykalin2008value}, this expression is equivalent to the Value-at-Risk (VaR) definition:
\begin{equation}
\operatorname{VaR}_\beta(h_{k+1}) = \sup_{\zeta \in \mathbb{R}} \{ \zeta \mid \mathbb{P}(h_{k+1} \geq \zeta) \geq 1-\beta \}, 
\end{equation}
where \(\zeta\) is a decision variable. It measures the $\beta$-quantile value of a random variable $h_{k+1}$.
% However, the above two constraints are often computationally intractable~\cite{rockafellar2000optimization}. 
% CVaR has been developed as a tight approximation of VaR.
We then define the CVaR:
\begin{definition}[Conditional Value-at-Risk~(CVaR) \cite{sarykalin2008value}]
The expected loss in the $\beta$-tail of a random variable $h_{k+1}$, given the threshold VaR$_\beta$ is described as:
\begin{equation}
\label{eq:cvar}
\operatorname{CVaR}_\beta(h_{k+1}) \coloneqq \mathbb{E}[h_{k+1} \mid h_{k+1} \leq \operatorname{VaR}_\beta(h_{k+1})].
\end{equation}
\end{definition}
As demonstrated in \cite{sarykalin2008value}, the expression in \eqref{eq:cvar} can be reformulated as the following optimization problem
\begin{equation}
\label{eq:cvar_opt}
\operatorname{CVaR}_\beta(h_{k+1}) =- \inf_{\zeta \in \mathbb{R}} \mathbb{E} \left[ \zeta + \frac{(-h_{k+1} - \zeta)_+}{\beta} \right],
\end{equation}
where $(\cdot)_+ = \max\{\cdot, 0\}$. Note, a value of $\beta \to 1$ corresponds to a risk-neutral case, i.e., $\operatorname{CVaR}_{\beta \to 1}(h_{k+1}) = \mathbb{E}(h_{k+1})$; whereas a value of $\beta \to 0$ is a risk-averse case, i.e., $\operatorname{CVaR}_{\beta \to 0}(h_{k+1}) = \operatorname{VaR}_{\beta \to 0}(h_{k+1})$ \cite{akella2024risk}. 
Compared with VaR, CVaR adheres to a group of axioms crucial for rational risk assessment \cite{majumdar_how_2020}.

% $\beta \to 1$ 
% $\beta \to 0$ 
Now, the connection between the probabilistic constraint, VaR, and CVaR constraint is:
\begin{equation}
\begin{aligned}
\label{eq:psafe}
& \operatorname{CVaR}_{\beta}(h_{k+1}) \geq 0 \ \Rightarrow \\
&\quad \operatorname{VaR}_{\beta}(h_{k+1}) \geq 0 \ \Leftrightarrow 
\ \mathbb{P}\bigl(h_{k+1}\geq 0\bigr) \geq 1 - \beta. 
\end{aligned}
\end{equation}

% $p(h_{k+1})$
\subsection{Discrete-Time CVaR Barrier Functions} 
To handle uncertainty in a risk-aware manner, we employ a dynamic coherent risk measure known as CVaR-Safety \cite{ahmadi2021risk}. Define the cumulative CVaR values from time $0$ to $k$ as the
composition of per-step CVaR operators applied sequentially over the time horizon. Formally, it is defined as:
\[\operatorname{CVaR}_{\beta}^{0:k} := \operatorname{CVaR}_{\beta}^0 \circ \operatorname{CVaR}_{\beta}^1 \circ \dots \circ \operatorname{CVaR}_{\beta}^k.
\]
\begin{definition}[CVaR-Safety \cite{ahmadi2021risk}]
Given a safe set \(\mathcal{S}\) as defined in \eqref{eq:safeset} and a risk level \(\beta \in (0, 1)\), we call the solutions to \eqref{eq:sys}, starting at \(\mathbf{x}_0 \in \mathcal{S}\), \emph{CVaR-safe} if 
\begin{equation}    
\operatorname{CVaR}_{\beta}^{0:k}(h_k) \geq 0, \quad \forall\, k \geq 0.
\end{equation}
\end{definition}

To enforce CVaR-safety, we utilize \emph{CVaR barrier functions} for discrete-time systems.

\begin{definition}[CVaR Barrier Functions \cite{ahmadi2021risk}]
For the discrete-time system \eqref{eq:sys} and a risk level \(\beta \in (0,1)\), a continuous function \(h: \mathbb{R}^n \to \mathbb{R}\) is called a \emph{CVaR barrier function} for the safe set \(\mathcal{S}\)~\eqref{eq:safeset} if there exists a constant \(\gamma \in (0,1]\) such that
for each $\mathbf{x}_k \in \mathbb{R}^n$, there exist a $\mathbf{u}_k \in \mathbb{R}^m$ s.t.:
\begin{equation}
\label{eq:cvar_bc}
\operatorname{CVaR}_{\beta}^k\bigl(h_{k+1}\bigr) \ge (1-\gamma)\, h_k,
\quad \forall\, \mathbf{x}_k \in \mathcal{X}.
\end{equation}
\end{definition}

\begin{theorem}[\cite{ahmadi2021risk}] 
Consider the discrete-time system in \eqref{eq:sys} and the safe set \(\mathcal{S}\) as defined in \eqref{eq:safeset}. Let \(\beta \in (0, 1)\) be a given confidence level. Then, \(\mathcal{S}\) is CVaR-safe if there exists a CVaR barrier function as defined above.
\label{thm:cvar_safe}
\end{theorem}
% \begin{proof}
% The proof can be found in \cite{ahmadi2021risk}.
% \end{proof}

\begin{figure}
    \centering
    \includegraphics[width=0.9\linewidth]{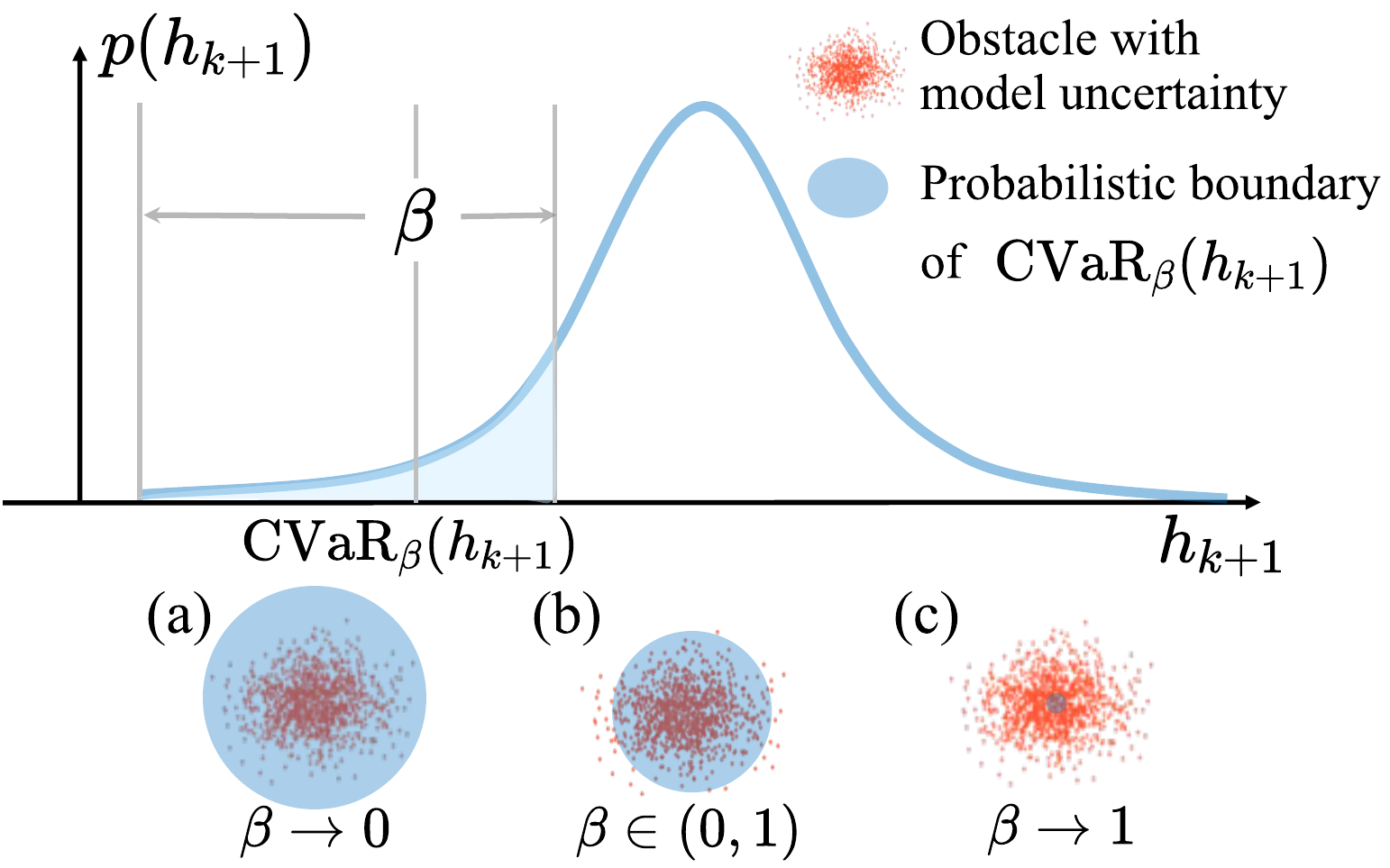}
    \caption{Illustration of relationship between CVaR, the risk level $\beta$, and safety. The shaded blue region represents the boundary of $\operatorname{CVaR}_{\beta}(h_{k+1})$, which includes: (a) all possible obstacle positions, (b) the portion of the distribution captured by the CVaR at a specified risk level $\beta$, and (c) the expected (mean) positions of the obstacles. Here, $p(h_{k+1})$ denotes the probability distribution of the $h_{k+1}$.
     } 
    \label{fig:risk}
    \vspace{-12pt}
\end{figure}

Similar to the constraint function of the CBF for deterministic systems in \eqref{eq:qp}, the optimization problem using the CVaR-BF constraint can be constructed as
\begin{equation}
\begin{aligned}
% \label{eq:cvarbc_opt}
\min_{\mathbf{u}_k \in \mathcal{U}} \|\mathbf{u}_k - \bar{\mathbf{u}}_k\|^2, 
\, \text{s.t.} \, \operatorname{CVaR}_{\beta}^k\bigl(h_{k+1}\bigr) \ge (1-\gamma)\, h_k. \\
% & \mathbf{u}_k \in \mathcal{U}.
\end{aligned}
\label{eq:cvarqp}
\end{equation}

\section{Adaptive Risk Level of CVaR-BF}
\subsection{Safety and Feasibility Analysis}

We identify a key issue in CVaR-BF constraint: the tuning hyper-parameter, the risk level~$\beta$, presents a trade-off between the safety and feasibility of the optimization problem in \eqref{eq:cvarqp}. 
Specifically, this risk level~$\beta$, which is set by the user and held constant throughout the trajectory \cite{ahmadi2021risk}, plays a critical role in determining the robot’s behavior near the boundary of the safe set.  
We analyze this issue from a set-based perspective in two aspects: safety and feasibility.

% , which may either lead to infeasibility when it is too conservative, or compromise safety when it is too relaxed.
% This observation highlights the inherent challenge in hard-constrained optimization problems: ensuring feasibility while maintaining stringent safety standards. 
% To better understand this issue, we analyze the problem from a set-based perspective in terms of safety and feasibility.

% In this setup, the parameter \(\beta\) directly influences how close the robot’s state trajectory can approach the safe set boundary. Increasing \(\beta\) relaxes the constraint, allowing the robot to operate closer to obstacles—albeit at the expense of increased risk—while decreasing \(\beta\) enforces stricter safety margins. A fixed choice of \(\beta\) may either lead to infeasibility when it is too conservative or compromise safety when it is too relaxed.
% This observation highlights the inherent challenge in hard-constrained optimization problems: ensuring feasibility while maintaining stringent safety standards. To better understand and address this issue, we analyze the problem from a set-based perspective by defining two key sets.

\textbf{Safety Analysis:} The CVaR-BF constraint ensures that the probabilistic constraint in \eqref{eq:psafe} is satisfied.
As shown in Fig.~\ref{fig:risk}: (1) higher \(\beta\) relaxes the CVaR-BF constraint, allowing the robot to operate closer to obstacles, albeit at the expense of increased risk; (2) lower \(\beta\) enforces a tighter CVaR-BF constraint, yet may cause a higher chance of infeasibility and over-conservative decisions.

\textbf{Feasibility Analysis:} Define the overall feasible set at each time step \(k\) as the intersection of the reachable set and the CVaR-BF constraint set. Specifically, the reachable set from the current state \(\mathbf{x}_k\) to the next state at time \(k+1\) is given by
\begin{equation}
\mathcal{R}_k = \left\{ \mathbf{x}_{k+1} \in \mathcal{X} \,\middle|\, \exists\, \mathbf{u}_k \in \mathcal{U} \text{ s.t. } \mathbf{x}_{k+1} = f(\mathbf{x}_k, \mathbf{u}_k) \right\}.
\end{equation}
At a given risk level \(\beta\), the set of admissible control inputs that satisfy the CVaR-BF constraint at time $k$ is defined as
\begin{equation}
\mathcal{U}^k_{\beta} = \left\{ \mathbf{u}_k \in \mathcal{U} \,\middle|\, \operatorname{CVaR}_{\beta}^k \bigl(h_{k+1}\bigr) \ge (1-\gamma)\, h_k \right\}.
\end{equation}
Consequently, the set of states for which the CVaR-BF constraint is feasible at time step \(k\) is
\begin{equation}
\mathcal{S}_{\operatorname{CVaR},k} = \left\{ \mathbf{x}_k \in \mathcal{X} \,\middle|\, \mathcal{U}^k_{\beta} \neq \emptyset \right\}.
\end{equation}
Thus, the overall feasible set is formulated as
\begin{equation}
\mathcal{F}_k = \mathcal{R}_k \cap \mathcal{S}_{\operatorname{CVaR},k}.
\label{eq:feasible_set}
\end{equation}
Notably, since the CVaR is monotonically increasing with respect to \(\beta\) according to its definition in \eqref{eq:cvar}, increasing \(\beta\) relaxes the safety constraint and consequently enlarges the feasible set. This analysis reveals a trade-off: if the safety constraints are too strict (small \(\beta\)), the intersection \(\mathcal{F}_k\) may be empty, leading to infeasibility; if they are too loose (large \(\beta\)), safety may be compromised. 

\subsection{Adaptive Risk Level}
\label{sec:beta}

To resolve this conflict, we propose an adaptive tuning strategy that at each time step $k$ adjusts \(\beta\) based on the robot’s perceived risk relative to the obstacles. This adaptive mechanism aims to maintain a proper balance between feasibility and safety.

Let us define the adaptive risk level at each time \(k\) as
\begin{equation}
\label{eq:betak1}
\beta_k \coloneqq \min\{ \beta \in (0,\beta_u] \mid \mathcal{U}^k_{\beta} \neq \emptyset \},
\end{equation}
where $\beta_u$ is the fixed risk level used in the standard CVaR formulation as in \eqref{eq:cvar} and here we use it as the upper bound of the adaptive risk level. In words, we adaptively select $\beta_k$ as the smallest value, not exceeding $\beta_u$, that yields a nonempty control space.

\begin{definition}[Risk Adaptive CVaR Barrier Function] \label{def:adaptive_cvar_bf}
Consider the discrete-time system \eqref{eq:sys} and an adaptive risk level \(\beta_k\) at each time step \(k\) as defined in \eqref{eq:betak1}. A function \(h:\mathbb{R}^n\to\mathbb{R}\) is called a \emph{Risk Adaptive CVaR Barrier Function} for the safe set \(\mathcal{S}\) in \eqref{eq:safeset} if there exists a constant \(\gamma \in (0,1]\) such that for each $\mathbf{x}_k \in \mathbb{R}^n$, there exist a $\mathbf{u}_k \in \mathbb{R}^m$ s.t.,
\begin{equation}
\label{eq:cvar_cons_betak}
\operatorname{CVaR}_{\beta_k}^k \left( h_{k+1} \right) \ge (1-\gamma)\, h_k, \quad \forall \mathbf{x}_k \in \mathcal{X}. 
\end{equation}
\end{definition}
The notion of the risk adaptive CVaR barrier function allows to initialize the risk level with a conservative value (close to zero), and then incrementally increase it when necessary. A trade-off is needed only when the robot nears obstacles, while risk level remains low elsewhere to maintain a high probability of safety throughout the trajectory \footnote{We restrict the adaptive risk level to not exceed $\beta_u$, so that in the worst-case scenario the risk probability remains identical to that of the fixed-parameter formulation.}.

\begin{theorem}[CVaR-Safety with Adaptive Risk Level]
Consider the discrete-time system \eqref{eq:sys} and the safe set  \(\mathcal{S}\)~\eqref{eq:safeset}. Let $\beta_u \in (0,1)$ be a fixed upper-bound risk level, and let $\beta_k \in (0,\beta_u]$ be an adaptive risk level at time $k$ as defined in \eqref{eq:betak1}. Then, \(\mathcal{S}\) is at least CVaR-safe with respect to the risk level $\beta_u$ if there exists a risk adaptive CVaR barrier function as defined in Definition~\ref{def:adaptive_cvar_bf}.
\label{thm:adaptive}
\end{theorem}

\begin{proof}
Since by construction $\beta_k \leq \beta_u$, and because CVaR is monotonic with respect to the risk level, it follows that $\operatorname{CVaR}_{\beta_u}^k(h_{k+1}) \ge \operatorname{CVaR}_{\beta_k}^k(h_{k+1})$. Thus, the condition $\operatorname{CVaR}_{\beta_k}^k(h_{k+1}) \ge (1-\gamma)h_k$ implies that $\operatorname{CVaR}_{\beta_u}^k(h_{k+1}) \ge (1-\gamma)h_k$. Together with Theorem~\ref{thm:cvar_safe}, this inequality guarantees that the solutions to \eqref{eq:sys} are CVaR-safe with a probability level that is at least as high as that ensured by a fixed risk level $\beta_u$.
\end{proof}

\begin{remark}
Given the maximum allowable risk level \(\beta_u\), there may still be cases where no solution exists for \eqref{eq:betak1}, particularly in dynamic obstacle environments.
\end{remark}

\section{Dynamic Zones for CVaR Barrier Functions}
\label{sec:adaptiveh}

%To address the above problem, in this section, we discuss how to further expand the adjustment space of the \(\beta_k\) value, while also ensuring that the given risk level $\beta_u$ is met; this is achieved with the notion of the dynamic zone-based barrier function.

We introduce the notion of a dynamic zone-based barrier function to expand the adjustment space of the \(\beta_k\) value, while also ensuring that the given risk level $\beta_u$ is met.

Effective risk management is essential for robotic navigation, particularly in environments with high-speed obstacles and uncertainties that shorten reaction times and increase collision risks.
Figure~\ref{fig:zone}a shows that using a fixed risk level can render the problem infeasible when the robot approaches an obstacle. In contrast, Fig.~\ref{fig:zone}b employs an adaptive risk level (as described in Sec.~\ref{sec:beta}) that starts conservatively prompting the robot to initiate a turn early, and then relaxes the safety requirements as needed. However, without incorporating a dynamic zone that provides a virtual radius (an extra buffer), this risk adjustment can ultimately compromise safety and lead to collisions.
Therefore, as shown in Fig.~\ref{fig:zone}c, the combination of an adaptive risk level with a dynamic zone not only facilitates early obstacle avoidance but also provides greater flexibility for risk adjustments, thereby relaxing constraints and expanding the feasible space.

\subsection{Dynamic Zone–Based Barrier Function}
% % Consider a robot modeled via double-integrator dynamics
% \begin{equation}
% \label{eq:dyn}
% \mathbf{x}_{k+1} = A\,\mathbf{x}_k + B\,\mathbf{u}_k,
% \end{equation}
% with
% \( 
% A = \begin{pmatrix} I & \Delta t\, I \\ \mathbf{0} & I \end{pmatrix}, \quad B = \begin{pmatrix} \tfrac{1}{2}\Delta t^2\, I \\ \Delta t\, I \end{pmatrix}.
% \)

% for xxx model distance valid
% Distance-Based Barrier Function:
Classical CBFs typically rely on a distance-based measure that fails to account for the motion and unpredictability of obstacles. For example, one common formulation is:
\begin{equation}
\begin{aligned}
\label{eq:hd}
h^\text{D}_{k} = \|\mathbf{p}_k - \xinyi{\hat{\mathbf{p}}^o_k} \|^2 - R_{\text{safe}}^2,
\end{aligned}
\end{equation}
where \(R_{\text{safe}}\) represents a threshold distance capturing the minimal allowable separation, \xinyi{and \(\mathbf{p}_k \in \mathbb{R}^{d}\) and \(\xinyi{\hat{\mathbf{p}}^o_k} \in \mathbb{R}^{d}\) represent the position of robot and obstacle w.r.t. a global frame (with \(n = 2d\) for a \(d\)-dimensional space), respectively.}
This formulation may lead to myopic behavior, causing the robot to navigate too close to obstacles and thereby increasing the risk of collision.

In contrast, functions based on velocity obstacles or collision cones \cite{tayal2024control, roncero2024multi} incorporate the relative motion between the robot and the obstacle. A typical candidate is:
\begin{equation}
\begin{aligned}
\label{eq:cone}
& h^\text{C}_{k} = \langle \mathbf{p}^\text{rel}_{k}, \mathbf{v}^\text{rel}_{k} \rangle + \|\mathbf{p}^\text{rel}_{k}\|\|\mathbf{v}^\text{rel}_{k}\| \cos \phi,\\
&\mathbf{p}^\text{rel}_{k} = \mathbf{p}_k - \xinyi{\hat{\mathbf{p}}^o_k}, \quad \mathbf{v}^\text{rel}_{k} = \mathbf{v}_k - \xinyi{\hat{\mathbf{v}}^o_k},
\end{aligned}
\end{equation}
where \(\phi\) is the half-angle of the collision cone, defined as
\(\cos \phi = \sqrt{\|\mathbf{p}^\text{rel}_{k}\|^2 - R_{\text{safe}}^2} \,/\, \|\mathbf{p}^\text{rel}_{k}\|\), \xinyi{and \(\mathbf{v}_k \in \mathbb{R}^{d}\) and \(\xinyi{\hat{\mathbf{v}}^o_k} \in \mathbb{R}^{d}\) represent the velocity of robot and obstacle, respectively}.
This CBF enforces that the angle between \(\mathbf{p}^\text{rel}_{k}\) and \(\mathbf{v}^\text{rel}_{k}\) remains less than \(\pi - \phi\), thereby ensuring that the robot is directed away from the obstacle.
However, such geometric constraints can be overly conservative. In highly dynamic environments, they may force the robot to take unnecessary actions, leading to premature path diversions and, in some cases, rendering the navigation problem infeasible~\cite{roncero2024multi}.

To address these limitations, we propose a dynamic zone–based barrier function that leverages the predicted relative state:
\begin{equation}
\begin{aligned}
\label{eq:hz}
    h^\text{Z}_{k} &\coloneqq \|\mathbf{p}_k - \xinyi{\hat{\mathbf{p}}^o_k}\|^2 - R_{\text{safe}}^2 \Bigl(1 + \Delta_k\Bigr), \\
    \Delta_k &= \left(-\frac{\langle \mathbf{p}^\text{rel}_{k}, \mathbf{v}^\text{rel}_{k} \rangle}{\|\mathbf{p}^\text{rel}_{k}\|\,\|\mathbf{v}^\text{rel}_{k}\|}\right)_+,
        % \Delta_k &= \operatorname{softplus}\!\left(-\frac{\langle \mathbf{p}^\text{rel}_{k},\, \mathbf{v}^\text{rel}_{k} \rangle}{\|\mathbf{p}^\text{rel}_{k}\|\,\|\mathbf{v}^\text{rel}_{k}\|}\right),
\end{aligned}
\end{equation}
where \((\cdot)_+\) denotes the nonnegative part, i.e., \(\max\{0, \cdot\}\), ensuring that
\(
\Delta_k \in [0,1]
\) \footnote{To facilitate the optimization, we employ the softplus function as an approximation of the max operator to get nonzero gradient \cite{zheng2015improving}.}.
The interpretation is as follows:
\begin{itemize}
    \item If \( \frac{\langle \mathbf{p}^\text{rel}_{k}, \mathbf{v}^\text{rel}_{k} \rangle}{\|\mathbf{p}^\text{rel}_{k}\|\,\|\mathbf{v}^\text{rel}_{k}\|} \geq 0\) $\Rightarrow$ \(\Delta_k = 0\). This implies that the robot and the obstacle are moving away from each other, thereby reducing the likelihood of a collision.
    \item If \( \frac{\langle \mathbf{p}^\text{rel}_{k}, \mathbf{v}^\text{rel}_{k} \rangle}{\|\mathbf{p}^\text{rel}_{k}\|\,\|\mathbf{v}^\text{rel}_{k}\|} < 0 \) $\Rightarrow$  \(\Delta_k > 0\). This implies that the  robot and the obstacle are approaching each other, thereby increasing the likelihood of a collision.
\end{itemize}

Instead of imposing a direct constraint on the relative angle which can lead to unnecessary obstacle avoidance when the robot is far away, our approach modulates the safety zone only when necessary. When the robot and obstacles are far apart, even if the safety zone radius is expanded, it does not significantly influence the robot's behavior due to the large relative distance. 
Thus, this strategy prevents unnecessary avoidance of obstacles and avoids the overly conservative behavior that can result from rigid angle constraints.

%  where the softplus function is defined as
% \(
% \operatorname{softplus}(x) = \ln\bigl(1+\exp(x)\bigr),
% \)
% and it provides a smooth approximation to \(\max\{0, x\}\)..

\subsection{Probabilistic Safety Guarantee} 
% In this section, we will give details how to compute the new upper bound $\bar{\beta}_u$ for the adaptive risk level within this dynamic zone–based approach. 
% The key insight is that the dynamic zone represents a (larger) dynamic yet virtual safety distance, rather than the actual physical distance between the robot and an obstacle (see in Fig.~\ref{fig:zone}c). 
% This enlarged safety zone provides an additional buffer, allowing us to use a wider range (denoted by \(\bar{\beta}_u\)) to tune the risk level while ensuring to maintain at least the same physical safety guarantee as that achieved with a lower, more conservative risk level \(\beta_u\) in the distance-based function.

Next, we detail how to derive a new upper bound, denoted by $\bar{\beta}_u \in (0,1)$, for the adaptive risk level within the dynamic zone–based approach. The key insight is that the dynamic zone represents a larger, dynamic, yet virtual safety distance, rather than the actual physical separation between the robot and an obstacle (see Fig.~\ref{fig:zone}c). We propose an analytical formulation that leverages this expanded safety zone to permit a wider range of risk level adaptation, i.e., $\beta_k \in (0, \bar{\beta}_u]$,  while preserving the same probabilistic safety guarantee as that achieved with the original conservative risk level $\beta_u$ associated with the conventional distance-based barrier function $h^\text{D}_{k}$~\eqref{eq:hd}.

Let \(h^\text{D}_{k+1}\) denote the original distance‐based function with probability density function $p_{h^\text{D}_{k+1}}(x)$, where $x$ represents a possible realization of $h^\text{D}_{k+1}$. Recall that our dynamic zone‐based barrier function $h^\text{Z}_{k+1}$ is defined by expanding the safety radius via a factor related to $\Delta_{k+1}$; equivalently, we express it as a shift:
\begin{equation}
\label{eq:hzhd}
h^\text{Z}_{k+1} = h^\text{D}_{k+1} - \Delta_{k+1} \, R_{\text{safe}}^2 .
\end{equation}
% Since this shift is deterministic, it does not change the shape of the distribution of \( h^\text{D}_{k+1} \). Thus, the probability density of $h^\text{Z}_{k+1}$ satisfies
% \begin{equation}
% p_{h^\text{Z}_{k+1}}(x) = p_{h^\text{D}_{k+1}}\Bigl(x + \Delta_{k+1} \, R_{\text{safe}}^2\Bigr).
% \end{equation}
By the subadditivity property of CVaR~\cite{majumdar_how_2020}, we have
\begin{equation}
\label{eq:cvarhzhd_betau}
\begin{aligned}
\operatorname{CVaR}^k_{\beta}\bigl(h^\text{Z}_{k+1}\bigr) &
= \operatorname{CVaR}^k_{\beta}\bigl(h^\text{D}_{k+1} - \Delta_{k+1}\,R_{\text{safe}}^2 \bigr)\\
& \leq \operatorname{CVaR}^k_{\beta}\bigl(h^\text{D}_{k+1}\bigr) - \operatorname{CVaR}^k_{\beta}\bigl(\Delta_{k+1}\,R_{\text{safe}}^2 \bigr).
\end{aligned}
\end{equation}
Substituting these relations \eqref{eq:hzhd} and \eqref{eq:cvarhzhd_betau} into the original CVaR-BF constraint yields:
\begin{equation}
\begin{aligned}
 \operatorname{CVaR}^k_{\bar{\beta}_u}\bigl(h^\text{Z}_{k+1}\bigr) &\ge (1-\gamma)\,h^\text{Z}_k \Rightarrow  \\
 \operatorname{CVaR}^k_{\bar{\beta}_u}\bigl(h^\text{D}_{k+1}\bigr) &\ge (1-\gamma)h^\text{D}_{k} \\&+\bigl(\operatorname{CVaR}^k_{\beta}\bigl(\Delta_{k+1}\bigr)- (1-\gamma)\Delta_{k}\bigr)R_{\text{safe}}^2.
\end{aligned}
\end{equation}
Define a dynamic risk offset at time $k+1$ as $\delta_{k+1} \coloneqq \bigl(\operatorname{CVaR}^k_{\beta}\bigl(\Delta_{k+1}\bigr)- (1-\gamma)\,\Delta_{k}\,\bigr) R_{\text{safe}}^2$.
When an obstacle moves toward the robot, causing the likelihood of a collision to increase, i.e., $\operatorname{CVaR}^k_{\beta}\bigl(\Delta_{k+1}\bigr) \geq \mathbb{E}(\Delta_{k+1}) > (1-\gamma) \Delta_{k}$, this implies that $\delta_{k+1}>0$.
% Define a dynamic risk offset at time $k+1$ as $\delta_{k+1} \coloneqq \Delta_{k+1}- (1-\gamma)\,\Delta_{k}\,R_{\text{safe}}^2$.
Since CVaR is monotonic with respect to the risk level, this additional term allows us to define a new, less conservative upper bound on the risk level.
% $\bar{\beta}_u$\footnote{ 
% The expanding radius increases the actual distance between robot and obstacles, broadening the range for the risk parameter \(\beta\) (i.e., making the upper bound \(\bar{\beta}_u\) less conservative). When the actual distance equals \(R_{\textbf{safe}}\), $h^\text{Z}$ is equivalent to $h^\text{D}$, thus the upper bound for \(\beta\) reverts to that of the original formulation.
% % Note that when an obstacle moves toward the robot, likelihood of a collision will increase i.e., $\Delta_{k+1} > (1-\gamma) \Delta_{k}$, which implies that $\bar{\beta}_u$ is less conservative.
% }.
One can show that, if the same safety threshold $R_\text{safe}$ is maintained, the risk level $\bar{\beta}_u$ permitted in the dynamic zone-based formulation can be greater than that in the original distance-based formulation $\beta_u$, i.e., $\bar{\beta}_u \geq \beta_u$.
The maximum risk level $\bar{\beta}_u$ can be computed by searching for the value such that the following condition holds:
% \begin{equation}
% \begin{aligned}
% \label{eq:cvar_relation}
% \operatorname{CVaR}^k_{\bar{\beta}_u}\bigl(h^\text{D}_{k+1}\bigr) = & \operatorname{CVaR}^k_{\beta_u}\bigl(h^\text{D}_{k+1}\bigr) + \delta_{k+1},
% % \operatorname{CVaR}^k_{\bar{\beta}_u}\bigl(h^\text{D}_{k+1}\bigr) = \operatorname{CVaR}^k_{\beta_u}\bigl(h^\text{D}_{k+1}\bigr) + \,\delta_{k+1},
% \end{aligned}
% \end{equation}
\begin{equation}
\label{eq:cvar_relation}
\begin{aligned}
    \bar{\beta}_u = \{ \beta \,| \operatorname{CVaR}^k_\beta(h_{k+1}^D) = \operatorname{CVaR}^k_{\beta_u}(h_{k+1}^D) + \delta_{k+1} \}
\end{aligned}
\end{equation}
% we can solve for $\bar{\beta}_u$ by inversion.
 Thus, by adopting the dynamic zone-based barrier function $h^\text{Z}_{k+1}$ and selecting the adaptive risk level according to
\begin{equation}
\label{eq:beta_k}
\beta_k = \min\left\{ \beta \in (0,\bar{\beta}_u] \,\middle|\, \mathcal{U}^k_{\beta} \neq \emptyset \right\},
\end{equation}
we ensure that the probabilistic safety guarantee is maintained at the level corresponding to the original $\beta_u$.

\begin{lemma}[Equivalence of Probabilistic Safety Guarantee]\label{thm:equivalence}
Given a safe set $\mathcal{S}$~\eqref{eq:safeset} that is CVaR-safe under the risk adaptive CVaR-BF defined in Definition~\ref{def:adaptive_cvar_bf} with the distance-based barrier function $h^\text{D}$ and a fixed risk level upper bound $\beta_u$. Then, by adopting the dynamic zone–based barrier function $h^\text{Z}$ together with the newly derived upper bound $\bar{\beta}_u$ for the adaptive risk level, the resulting safety guarantee is equivalent to that provided by the original CVaR-BF. In other words, the safe set $\mathcal{S}$ remains CVaR-safe with the same probabilistic guarantee.
\end{lemma}

\begin{proof}
Under the distance‐based formulation, the adaptive risk level $\beta_k$ is chosen from $(0,\beta_u]$ so that the risk adaptive CVaR-BF guarantees safety (by Theorem~\eqref{thm:adaptive}). In the dynamic zone–based formulation, $\beta_k$ is selected from $(0,\bar{\beta}_u]$. By design, when the adaptive risk level in the dynamic zone approach reaches its upper bound $\bar{\beta}_u$, the resulting CVaR condition is equivalent to that obtained with the upper bound $\beta_u$ in the distance‐based case. Therefore, the safe set $\mathcal{S}$ remains CVaR‐safe with at least the same probabilistic safety guarantee.
\end{proof}

% \begin{theorem}[CVaR-Safety with Adaptive Barrier Function]\label{thm:adaptive_cvar_safety}
% Consider the discrete-time system in \eqref{eq:sys} and the safe set $\mathcal{S}$ defined in \eqref{eq:safeset}. Let $\beta_u \in (0,1)$ be a given risk level. If there exists an adaptive CVaR barrier function satisfying
% \begin{equation}
% \label{eq:cvar_cons_betak_hz}
% \operatorname{CVaR}_{\beta_k}^k\bigl(h^\text{Z}_{k+1}\bigr) \ge (1-\gamma)\, h^\text{Z}_k, \quad \forall\, \mathbf{x}_k \in \mathcal{X},
% \end{equation}
% with the barrier function $h_{k+1}$ defined in \eqref{eq:hz}, and $\beta_k$ can be seleted according to
% \begin{equation}
% \label{eq:beta_k}
% \beta_k = \min\left\{ \beta \in (0,\bar{\beta}_u] \,\middle|\, \mathcal{U}^k_{\beta} \neq \emptyset \right\},
% \end{equation}
% then the safe set $\mathcal{S}$ is CVaR-safe. That is, the system trajectories remain within $\mathcal{S}$ with the pre-defined risk level  $\beta_u$.
% \end{theorem}

\subsection{Risk Adaptive CVaR-BF Optimization}
Building on the CVaR-BF constraint in \eqref{def:adaptive_cvar_bf} and leveraging the adaptive risk level \(\beta_k\) defined in \eqref{eq:beta_k}, we formulate an optimization problem that ensures probabilistic safety while minimally deviating from a nominal control input. Consider the robot operates in an environment with \(N\) \textit{dynamic, uncertain} obstacles. The objective is to design a controller that drives a mobile robot toward its target while avoiding dynamic obstacles, thereby achieving a ``reach and avoid" task with minimal collision probability. 

%expectation of random variable $h_{k+1}$

\begin{figure*}
    \centering
    \includegraphics[width=0.89\linewidth]{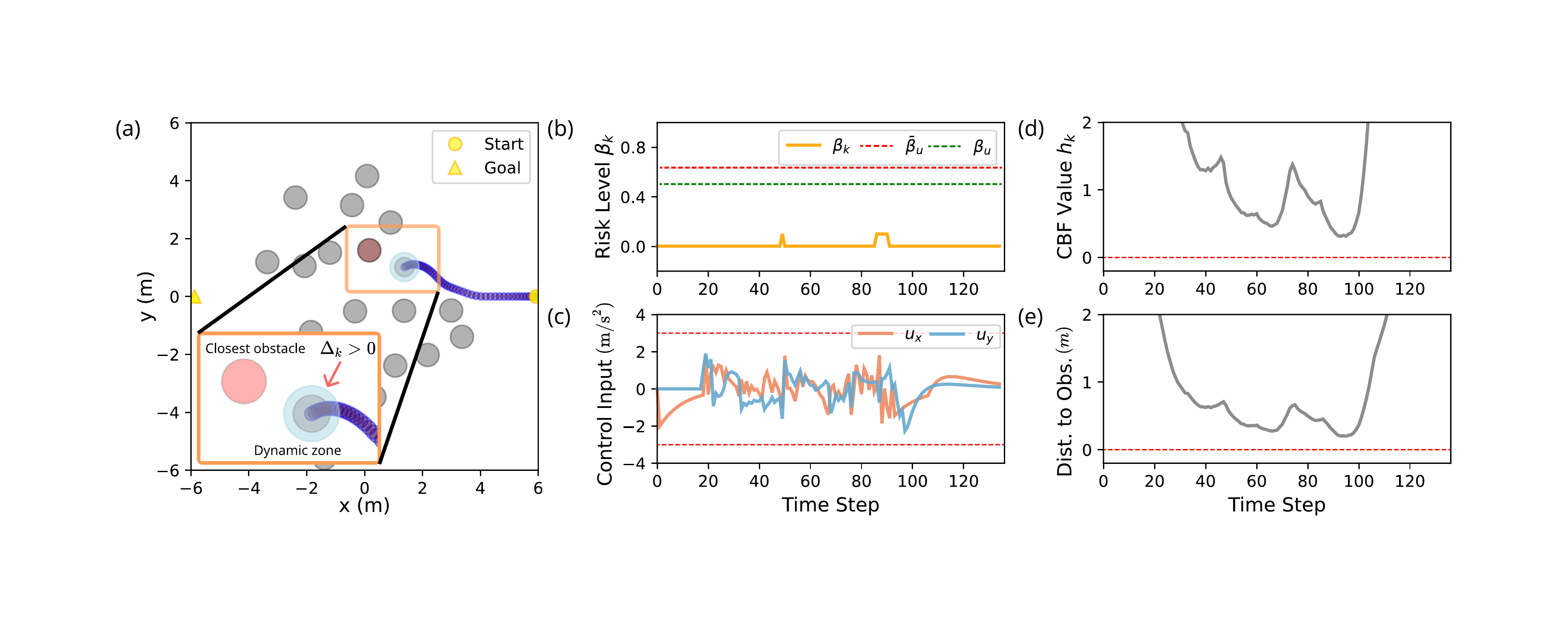}
    \caption{Visualization of robot trajectories and associated metrics for the whole trajectory in an environment with 20 uncooperative obstacles. 
    (a) Snapshot of the trajectory at a critical time step. (b) Risk level $\beta_k$ over time. (c) Control input over time. (d) CBF value over time (e) Distance to closest obstacle over time. 
    In (a), when the robot and an obstacle approach each other with high relative velocity, the dynamic safety margin will expand, i.e., $\Delta_k>0$ making robot proactively avoid obstacles. As shown in (b), $\beta_k$ will also increase where feasible space is limited, but it remains below the upper bound $\bar{\beta}_u$, ensuring safety without being overly conservative. (e) shows robot maintains a safe distance from obstacles due to dynamic zone-based barrier function.}
    \label{fig:sceanrio}
\vspace{-17pt}
\end{figure*}

\xinyi{
To approximate the CVaR term in the safety constraint, for each obstacle, 
we consider a finite set of possible realizations of the next obstacle state, $\mathcal{X}^o_{k+1} = (\mathbf{x}^{o,(1)}_{k+1}, \mathbf{x}^{o,(2)}_{k+1}, \dots, \mathbf{x}^{o,(L)}_{k+1})$, where $L$ denotes the total number of samples. Each realization is associated with a probability $p(\mathbf{x}^{o,(j)}_{k+1})$. 
These samples may be drawn from any predictive distribution or model that captures the uncertainty in the obstacle’s next state. For each realization, we compute the corresponding value of the random variable $h_{k+1}$, which represents the safety measure at time $k+1$. These samples are then used to approximate the expectation in the CVaR formulation in~\eqref{eq:cvar}.}
% To approximate the CVaR term in the safety constraint, we sample noise~$\mathbf{w}_k$ in a finite set $\mathcal{W} = \{\mathbf{w}^{(1)}, \mathbf{w}^{(2)}, \dots, \mathbf{w}^{(L)}\}$, where each \(\mathbf{w}^{(j)} \in \mathbb{R}^n\) occurs with probability $p\bigl(\mathbf{w}^{(j)}\bigr) \coloneqq P\Bigl(\mathbf{w}_k = \mathbf{w}^{(j)}\Bigr), \, j = 1, 2, \dots, L$. Assume \(\mathbf{w}_k\) is independent of the deterministic initial condition for all \(k\). For each noise sample, we compute the corresponding value of the random variable $h_{k+1}$, which represents the safety measure at time $k+1$. These samples are then used to approximate the expectation in the CVaR formulation in \eqref{eq:cvar}.
Using the CVaR definition from \eqref{eq:cvar} and its reformulation in \eqref{eq:cvarqp}, the following optimization problem is solved at each time $k$:
\begin{problem}[Risk Adaptive CVaR-BF Optimization]
\begin{equation}
\begin{aligned}  
\label{eq:opt2}
&\min_{\mathbf{u}_k \in \mathcal{U}, \zeta_i \in \mathbb{R}}
 \;\|\mathbf{u}_k - \bar{\mathbf{u}}_k\|^2 \\
 &\mathrm{s.t.}  -\Bigl( \zeta_i + \frac{1}{\beta_{k}} \sum_{j=1}^{ L} p_j \bigl[-h^\text{Z}_{i,j,{k+1}} - \zeta_i\bigr]_+ \Bigr) \geq (1-\gamma_{i})\, h^\text{Z}_{i,k},\\
 &\quad \quad \quad  \forall\, i \in \{1,\dots,N\}, \forall\, j \in \{1, \dots, L\}. \nonumber
\end{aligned}
\end{equation}
\end{problem}

By introducing auxiliary variables $\eta_j = h^\text{Z}_{i,j,{k+1}}-\zeta_i$, we can equivalently reformulate the max operators \((\cdot)_+\) as linear inequalities \cite{sarykalin2008value}:
\begin{equation}
\begin{aligned}  
\label{eq:opt3}
& \min_{\mathbf{u}_k \in \mathcal{U},  \zeta_i \in \mathbb{R}, \eta_j \in \mathbb{R}}
 \;\|\mathbf{u}_k - \bar{\mathbf{u}}_k\|^2 \\
 &  \text{s.t.}  \quad\quad \eta_j \geq 0,  \quad  -h^\text{Z}_{i,j,{k+1}}  - \zeta_i - \eta_j \leq 0, \\
 & \quad\quad -\Bigl(\zeta_i + \frac{1}{\beta_{k}} \sum_{j=1}^{L } p_j\, \eta_j \Bigr) \geq (1-\gamma_{i})\, h^\text{Z}_{i,{k}},\\
 &\quad\quad \quad  \forall\, i \in \{1,\dots,N\},\; \forall\, j = \{1, \dots, L\}.
\end{aligned}
\end{equation}

% \xinyi{
% \begin{remark}
% Feasibility condition: By setting $\beta = 1$, Problem~1 reduces to the deterministic CBF-based optimization with multiple constraints and input limits. Therefore, as long as the deterministic problem is feasible, Problem~1 is always guaranteed to be feasible by setting $\beta = 1$, even if no feasible solution exists for any $\beta < 1$.
% \end{remark}
% }

% \textit{Case Study: Safe Navigation for a Mobile Robot.}  
% The robot operates in an environment with \(N\) dynamic obstacles.
% The objective is to design a controller that drives a mobile robot toward its target while avoiding dynamic obstacles, thereby achieving a "reach and avoid" task with minimal collision probability. 
% % Our approach minimizes deviations from a nominal control input—one that does not explicitly account for safety—while enforcing safety through dynamic risk measures. 

\section{Simulations}
% \textcolor{red}{TODO}
\subsection{Experimental Setup}
% \xinyi{in \eqref{xx}, softplus to make it smooth }
\subsubsection{Implementation Details}
% We implement the proposed method in Python using CasADi \cite{andersson2019casadi}. 
We set $\beta_u = 0.5$ as desired risk level and $\bar{\beta}_u$ can be approximately estimated by numerically solving \eqref{eq:cvar_relation}. Then, to solve \eqref{eq:beta_k}, we discretize the interval \((0,\bar{\beta}_u]\) into a finite set of candidate risk levels with $B$ elements, $\mathcal{B} = \{ \beta^{(1)}, \beta^{(2)}, \dots, \beta^{(B)} \}$, with 
\(
0 < \beta^{(1)} < \beta^{(2)} < \cdots < \beta^{(B)} \le \bar{\beta}_u.
\)
For each $\beta \in \mathcal B$, we assess the feasibility of the optimization problem in \eqref{eq:opt3} using parallel processing. Specifically, at each time step \(k\), the adaptive risk level \(\beta_k\) is determined as:
\begin{equation}
\label{eq:betak2}
\beta_k = \min\{ \beta \in \mathcal{B} \mid \mathcal{U}^k_{\beta} \neq \emptyset \},
\end{equation}
 % to the Equ.~\eqref{eq:beta_adapt}. To accelerate the evaluation, the feasibility tests for different \(\beta^{(i)}\) can be performed concurrently using parallel processing.
% Although more sophisticated search techniques (e.g learning-based methods) can be employed \cite{kim2024learning}, our focus is on illustrating the core concept of adaptive risk tuning.

% The adaptive strategy starts by initializing with a conservative risk level and then incrementally increasing it as the system approaches obstacles. In our approach, we discretize the range of \(\beta\) values and use a search-based method—leveraging parallel processing—to select the optimal risk level in real time at each time step.

\subsubsection{Agent Settings}

\xinyi{
While, in general, the next positions of obstacles could be predicted from learned trajectory prediction models (e.g., \cite{huang2022survey}), for simplicity we adopt a constant-velocity model with additive uncertainty. Specifically, for each obstacle, we generate $L$ samples of the next position as $\mathbf{p}^o_{k+1} = \hat{\mathbf{p}}^o_{k} + \hat{\mathbf{v}}^o_{k} \Delta t + \mathbf{w}_{k}$, where each \(\mathbf{w}_{k} \in \mathcal{W}\) is assigned a probability over the position space. We set $L =20$ in our experiments.
Although the uncertainty set $\mathcal{W}$ can be constructed using any quantification technique (e.g., \cite{zhang2024safety}) and is not limited to a particular distribution, in this case study, we simply sample $\mathbf{w}_{k}$ from a zero-mean Gaussian with covariance \(
\Sigma_p = \operatorname{diag}(\sigma^2, \sigma^2) \). The standard deviation is varied as \(\sigma \in \{0.0,\, 0.025,\, 0.05,\, 0.075,\, 0.15\}\), and samples are drawn within \(\pm 3\sigma\) for each axis.}
\xinyi{
We assess our method in a widely used crowd navigation simulator within a 12m $\times$ 12m space \cite{chen2019crowd} (see Fig.~\ref{fig:sceanrio}), where each obstacle follows the Social Force Model (SFM)~\cite{helbing1995social} with \textit{uncooperative behavior}, meaning obstacles avoid collisions only with one another, not with the robot.} Their maximum speeds along each axis are chosen uniformly from $\{0.3, 0.6, 0.9, 1.2\}$m/s, and their radii are selected from $\{0.3, 0.4, 0.5\}$m. The robot is modeled as a double integrator, with dynamics given by
\begin{equation}
\begin{aligned}
\mathbf{p}_{k+1} &= \mathbf{p}_{k} + \Delta t\, \mathbf{v}_k + \tfrac{1}{2}\Delta t^2\, \mathbf{a}_k,  
% \nonumber 
\\ 
\mathbf{v}_k &= \mathbf{v}_k + \Delta t\, \mathbf{a}_k, 
% \nonumber
\end{aligned}
\end{equation}
\xinyi{
The nominal controller $\bar{\mathbf{a}}_k$ is designed as a simple go-to-goal proportional controller.
% \begin{equation}
% \begin{aligned}
% \bar{\mathbf{a}}_k = K_p \left( \mathbf{p}_{\text{goal}} - \mathbf{p}_k \right) - K_v \mathbf{v}_k.
% \end{aligned}
% \end{equation}
% where $K_p$ and $K_v$ are proportional and velocity gain coefficients, respectively. 
}
For each axis, the maximum acceleration is restricted to be less than \(3\,\text{m/s}^2\), and the maximum velocity is limited to under \(2\,\text{m/s}\).
The time step is set to $\Delta t = 0.1\,\text{s}$. The sensor range is \(5\,\text{m}\).

\begin{figure*}[t]
\centering
\hspace{-0.2cm}
\subfloat[SR vs. number of obstacles ($\sigma$ = 0.0)]{%
  \includegraphics[width=0.30\linewidth]{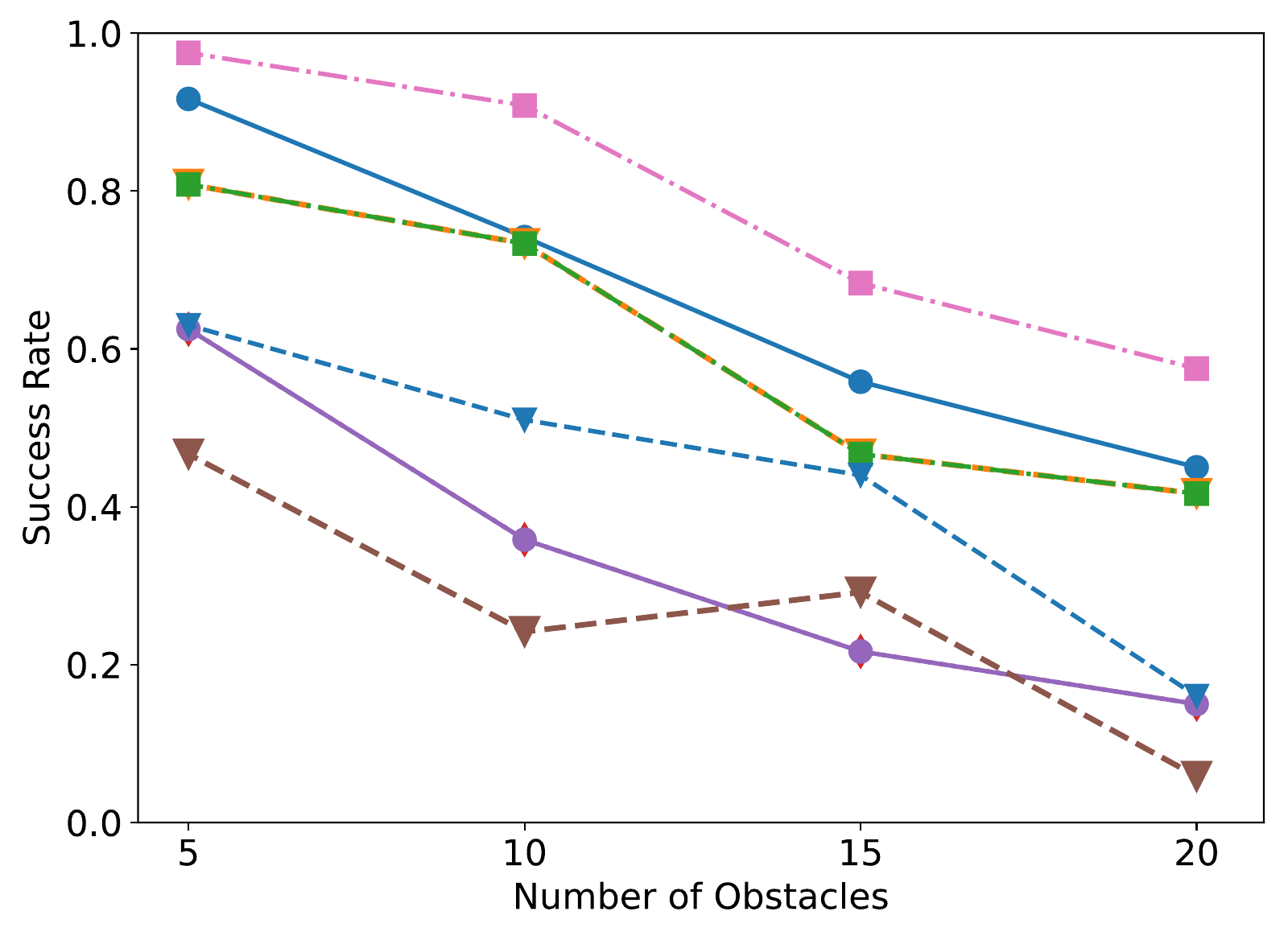}%
  \label{fig:a}
}
\hspace{-0.3cm}
\subfloat[SR vs. number of obstacles ($\sigma$ = 0.05)]{%
  \includegraphics[width=0.30\linewidth]{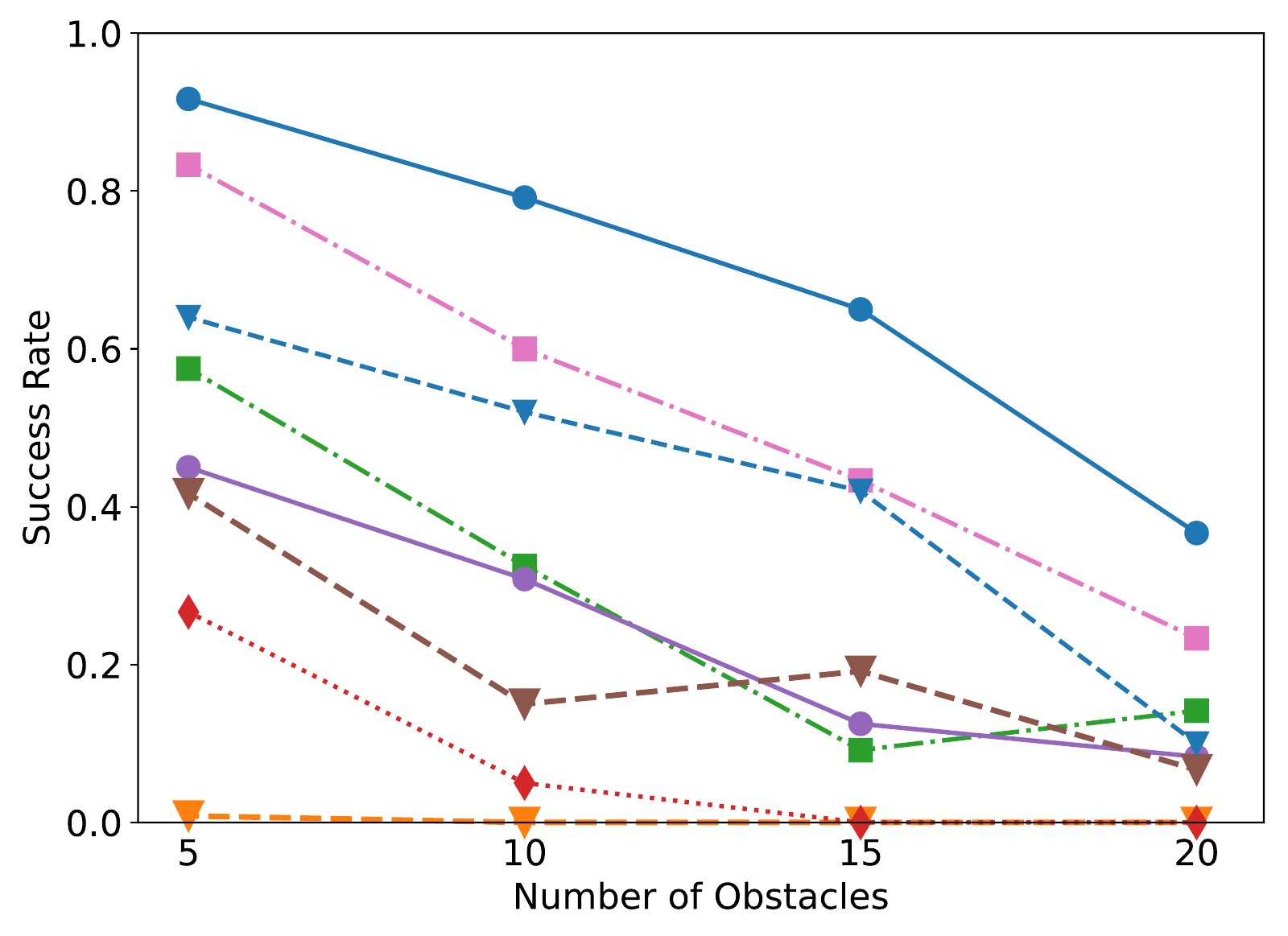}%
  \label{fig:b}%
}
\subfloat[SR vs. noise level (obstacle number = 15)]{%
  \includegraphics[width=0.30\linewidth]{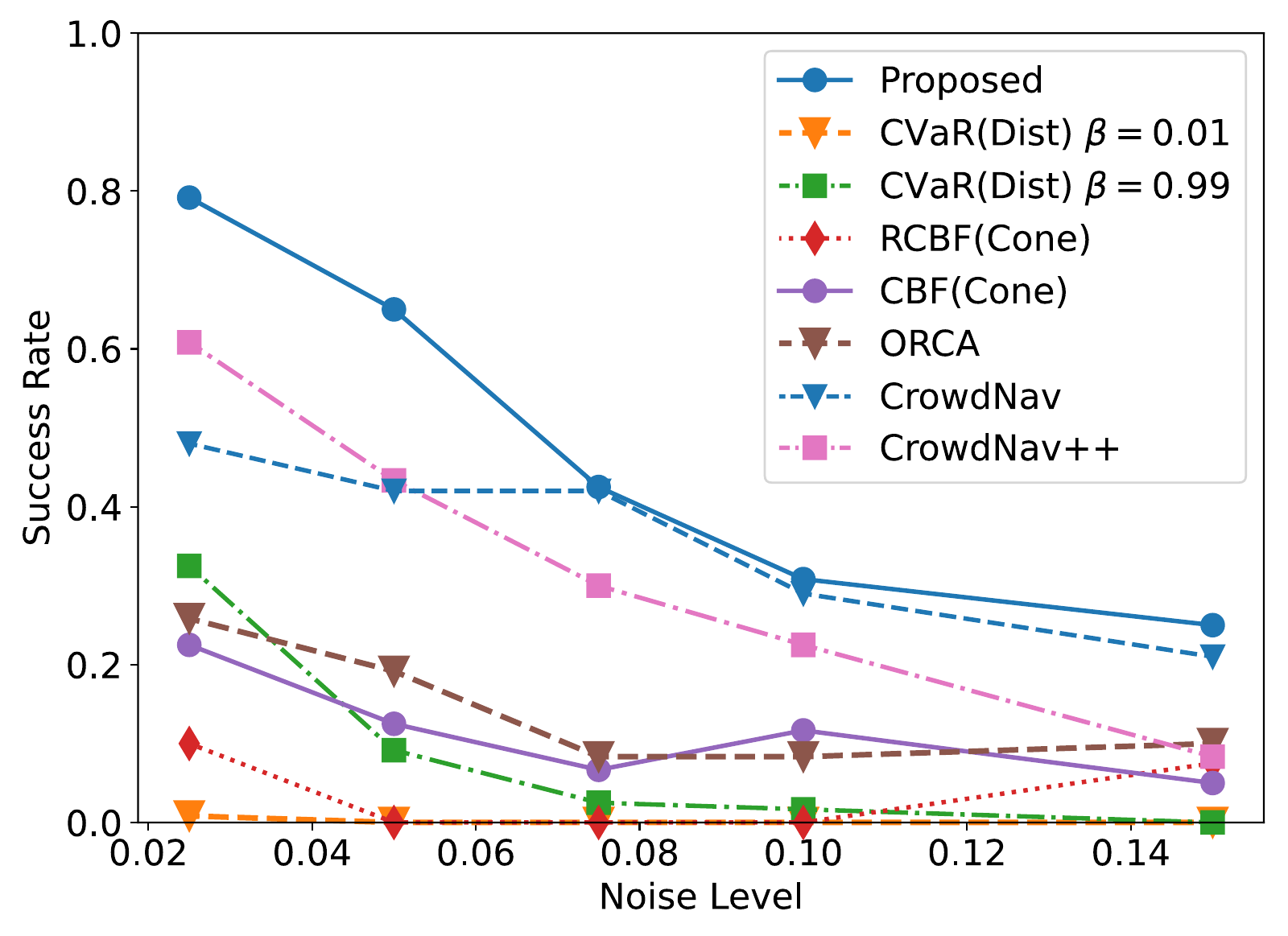}%
  \label{fig:c}%
}
\caption{Success rate comparisons under different noise levels and obstacle numbers.}
\label{fig:benchmark}
\vspace{-15pt}
\end{figure*}

\subsubsection{Performance Metrics}
Let \(m_t, m_s, m_f,\) and \(m_c\) denote the total, successful, and feasible test cases, and the number of collisions, respectively, and let $\mathcal{M}$ be the set of successful cases. 
A test case is considered \emph{feasible} if, for every time step \(k\) along the planned trajectory, the feasible set~$\mathcal{F}_k $~\eqref{eq:feasible_set} of the optimization problem in \eqref{eq:opt3} is nonempty, i.e., \(\forall k,\; \mathcal{F}_k \neq \emptyset\). Furthermore, a test case is deemed \emph{successful} if it is feasible and the robot reaches the goal point, denoted by \(\mathbf{x}_{\text{goal}}\), i.e., $\bigl(\forall k,\quad \mathcal{F}_k \neq \emptyset\bigr)
\, \wedge \,  \bigl(\exists\,K:\,\|\mathbf{x}_k - \mathbf{x}_{\text{goal}}\| < \varepsilon
, \,\forall k \ge K\bigr)$, where \(\varepsilon\) is a small positive threshold.
For each \(i \in \mathcal{M}\), \(L_i\) and \(T_i\) denote the trajectory length and execution time, respectively. We evaluate the following performance metrics: 1) Success Rate (SR): \(\displaystyle \frac{m_s}{m_t}\), 2) Feasibility Rate (FR): \(\displaystyle \frac{m_f}{m_t}\), 3) Collision Rate (CR): \(\displaystyle \frac{m_c}{m_t}\) 4) Average Trajectory Length (ATL): \(\displaystyle \frac{1}{m_s} \sum\nolimits_{i \in \mathcal{M}} L_i\), 5) Average Trajectory Time (ATT): \(\displaystyle \frac{1}{m_s} \sum\nolimits_{i \in \mathcal{M}} T_i\).

\subsection{Benchmark Comparisons}
We compare our method with the following baselines:
\begin{enumerate}
    \item CVaR-BF Methods: distance-based CVaR-BF with fixed risk levels (\(\beta=0.01\) or \(0.99\))~\cite{ahmadi2021risk} (\textit{CVaR(Dist)}).
    \item CBF Methods: collision cone-based CBF, as defined in~\eqref{eq:cone}~\cite{tayal2024control}, including a robust control for worst-case scenarios (\textit{RCBF(Cone)}) and a standard CBF-based control without explicit uncertainty handling (\textit{CBF(Cone)}).
    \item RL Methods: socially attentive reinforcement learning (\textit{CrowdNav})~\cite{chen2019crowd} and its extension incorporating predicted obstacle intentions (\textit{CrowdNav++})~\cite{liu2023intention}.
    
    \item Geometric Methods: reciprocal velocity obstacles for collision-free motion (\textit{ORCA})~\cite{alonso2013optimal}.
\end{enumerate}
Note that both ORCA and the RL methods assume a holonomic model, resulting in a simpler problem setup compared to the other methods. 
For all experiments, we average on 120 different configurations for the obstacles to get the results. 
% For RL methods, we use their released code to run experiments under our configurations, while for all other methods we build

\begin{table}[ht]
\centering
\caption{
 Comparison results on 5 obstacle scenario with noise $\sigma=0.05$
}
\label{tab:benchmark}
\begin{tabular}{lccccc}
\toprule
\textbf{Method}                & \textbf{SR}   & \textbf{FR} & \textbf{CR}  & \textbf{ATL} & \textbf{ATT} \\ 
\midrule
\rowcolor{gray!20}
\textbf{Proposed}                       & \textbf{0.917}   & \textbf{1.000}  & 0.058   & 13.036       & 13.092     \\
CVaR(Dist) $\beta=0.01$         & 0.008        & 0.008       & \textbf{0.000}   & 11.492       & 9.600      \\
CVaR(Dist) $\beta=0.99$         & 0.575         & 0.683     & 0.108     & 11.571       & 9.238      \\
RCBF(Cone)                     & 0.267         & 0.267      & \textbf{0.000}  & 11.697       & 8.175      \\
CBF(Cone)                      & 0.450         & 0.750       & 0.300       & 11.456       & 8.117      \\
ORCA                           & 0.417        & -          & 0.583      & \textbf{11.424}       & \textbf{6.822}      \\
CrowdNav                     & 0.642       & -          & 0.358     & 13.459       & 8.260     \\
CrowdNav++                     & 0.833        & -          & 0.167     & 14.266       & 8.451      \\
\bottomrule
\end{tabular}
\vspace{-5pt}
\end{table}

% \begin{table}[ht]
% \centering
% \caption{
% % Baseline comparison  with success rate, feasibility rate, collision rate, average trajectory length, and average trajectory time.
%  Comparison results on 5 obstacle scenario with noise $\sigma=0.05$
% }
% \label{tab:benchmark}
% \begin{tabular}{lccccc}
% \toprule
% \textbf{Method}                & \textbf{SR}   & \textbf{FR} & \textbf{CR}  & \textbf{ATL} & \textbf{ATT} \\ 
% \midrule
% Proposed                       & \textbf{0.917}   & \textbf{1.000}  & 0.058   & 13.036       & 13.092     \\
% CVaR(Dist) $\beta=0.01$         & 0.008        & 0.008       & \textbf{0.000}   & 11.492       & 9.600      \\
% CVaR(Dist) $\beta=0.99$         & 0.575         & 0.683     & 0.108     & 11.571       & 9.238      \\
% RCBF(Cone)                     & 0.267         & 0.267      & \textbf{0.000}  & 11.697       & 8.175      \\
% CBF(Cone)                      & 0.450         & 0.750       & 0.300       & 11.456       & 8.117      \\
% ORCA                           & 0.417        & -          & 0.583      & \textbf{11.424}       & \textbf{6.822}      \\
% CrowdNav                     & 0.642       & -          & 0.358     & 13.459       & 8.260     \\
% CrowdNav++                     & 0.833        & -          & 0.167     & 14.266       & 8.451      \\
% \bottomrule
% \end{tabular}
% % \vspace{-5pt}
% \end{table}

% First, we want to show our adaptive CVaR-BF helps improve the feasibility rate and reduce the colllision, thus enhance the success rate overall. 
% analyze the relation between safety and feasibility.
Table~\ref{tab:benchmark} presents the results for the scenario with $5$ obstacles and a noise level $\sigma$ = 0.05. 
% It shows that the proposed method outperforms the baseline approaches by achieving an optimal balance of high feasibility rates while maintaining a low collision rate. 
As we can see, our proposed method achieves the highest success rate by achieving the highest feasibility rate while keeping a low collision rate, demonstrating our adaptive CVaR-BF helps improve optimization feasibility and safe decision.
In contrast, the \textit{CVaR(Dist)} methods show limitations: with \(\beta=0.99\), the absence of a dynamic zone results in a higher collision rate, whereas with \(\beta=0.01\), the lack of adaptive risk adjustment significantly reduces the feasibility rate. 
The cone-based methods, \textit{RCBF(Cone)} and \textit{CBF(Cone)}, while yielding a lower collision rate, are overly conservative and also suffer from poor feasibility. Although \textit{CrowdNav++} achieves the second highest success rate, it lacks a probabilistic safety guarantees, leading to a high collision rate.
% Additionally, even though  and geometric methods struggle to handle obstacle uncertainty, failing to guarantee probabilistic safety.
% Besides, the RL method faces a generalization issue, as shown in xxx.
% This shortcoming contributes to their higher collision rates (see Fig.~\ref{fig:rl})

We provide a more detailed analysis of the statistical SR in Fig.~\ref{fig:benchmark} with respect to different number of obstacles and noise levels. 
% To ensure fairness, \textit{CrowdNav} and \textit{CrowdNav++} are evaluated in the same training environment used during its development. 
Fig.~\ref{fig:a} shows that in a deterministic environment ($\sigma = 0.0$), \textit{CrowdNav++} achieves the highest success rate since it is trained in the same deterministic environment. \textit{CrowdNav} lacks obstacle trajectory prediction, thus exhibiting a lower success rate.
Note that all CVaR-BF based methods reduce to standard CBF methods when $\sigma = 0.0$, leading to ineffective risk adjustments, therefore do not show superior performance. However, since our approach incorporates a dynamic zone, it still outperforms the cone-based and distance-based barrier functions.
As shown in Fig.~\ref{fig:b}, introducing noise reduces the SR of RL methods regardless of the number of obstacles. Although \textit{CrowdNav++} considers randomized obstacle behaviors (e.g., goal change and size variation), it still experiences a significant performance drop under uncertainty. 
More conservative methods, e.g., \textit{RCBF(Cone)} and \textit{CVaR(Dist)} with \(\beta=0.01\) also suffer significant degradation due to feasibility issues, resulting in the worst performance. 
An interesting finding is that our method performs even better in uncertain environments (Fig.~\ref{fig:b}) than in deterministic ones (Fig.~\ref{fig:a}). 
This is because in uncertain environments, the risk adaptation mechanism becomes effective and provides additional flexibility.
Fig.~\ref{fig:b} and Fig.~\ref{fig:c} further indicate that as noise level increases, the performance of all methods deteriorates. However, our proposed approach consistently achieves the highest success rate, demonstrating its robustness against uncertainty.

\subsection{Ablation Studies}
Table~\ref{tab:ablation} extends the experiments presented in Table~\ref{tab:benchmark} by using the same scenario, providing further insights on the contributions of individual components in our approach.

\subsubsection{Effectiveness Validation of Dynamic Zone CBF} 
We first replace the dynamic zone-based barrier function with other counterparts. The distance-based barrier function achieves a high feasibility rate but at the sacrifice of safety. Conversely, the cone-based barrier function is excessively conservative, leading to poor overall performance. 
These observations underscore the importance of the dynamic zone: 
even with risk adaptation, other CBF methods still struggle in such highly dynamic obstacle environments. 
% The dynamic zone provides the robot with sufficient space to effectively handle obstacles without resorting to overly cautious behavior.
% On the other hand
% The distance based barrier function, \textit{Dist}, achieves a perfect feasibility rate. However, it exhibits a notably higher collision rate, resulting in a poor performance compared to the proposed method. In contrast, while the collision cone based barrier function, \textit{Cone}, maintains a low collision rate, its  feasibility rates are significantly lower. 
% These results suggest that the Cone formulation may be overly conservative or less effective in complex scenarios.
\subsubsection{Effectiveness Validation of Risk Adaptation}
% Adjusting the risk parameter $\beta$ also impacts performance.
We further validate the effectiveness of adaptive risk levels by using fixed risk levels. A low $\beta$ value (0.01) yields a very low collision rate, but at the expense of  feasibility rates. Conversely, increasing  $\beta$ to 0.99 enhances the feasibility rate, yet with an increase in collisions. In contrast, our risk adaptation achieves great feasibility at the very low safety cost. 
% only permits more aggressive movements when a critical threshold is reached. 
% This indicates that relaxing the safety constraint (i.e., adopting a higher risk level) can improve overall performance by allowing more aggressive maneuvers.

\begin{table}[t]
\centering
\caption{Ablation study results: performance comparison of the proposed method and its variants using different CBFs and risk-level strategies.}
\label{tab:ablation}
\scalebox{0.88}{
\begin{tabular}{ccccccc}
\toprule
\textbf{Method}  & \textbf{Setting}  & \textbf{SR} & \textbf{CR} & \textbf{FR}  & \textbf{ATL} & \textbf{ATT} \\
\midrule
\rowcolor{gray!20}
\textbf{Proposed}     &  --     & \textbf{0.917}  & 0.058  & \textbf{1.000} & \textbf{13.036}       & 13.092 \\
\midrule
\multirow{2}{*}{\begin{tabular}{@{}c@{}} w/o \\ dynamic zone \end{tabular}}
  & Dist-based         & 0.859  & 0.133  & \textbf{1.000}  & 13.198  & 13.914 \\
  & Cone-based         & 0.548  & 0.067  & 0.615  & 13.529  & \textbf{11.793} \\
\midrule
\multirow{2}{*}{\begin{tabular}{@{}c@{}} w/o \\ risk adaptation \end{tabular}}
  & Fix $\beta$=0.01   & 0.681  & \textbf{0.007}  & 0.689  & 13.395  & 13.000 \\
  & Fix $\beta$=0.99   & 0.741  & 0.052  & 0.793  & 13.253  & 12.566 \\
\bottomrule
\end{tabular}}
\vspace{-15pt}
\end{table}
% \xinyi{
% \subsection{Real-World Dataset Validation}
% To further validate the practicality of our approach, we tested the proposed CVaR-BF optimization as a safety filter for RL-based method, e.g.,CrowdNav++, which was used as the nominal controller, on a real-world pedestrian trajectory dataset~\cite{xxx}. 
% Experimental results show that RL alone frequently failed to avoid collisions in complex scenarios. In contrast, when combined with our safety filter, the RL agent was able to navigate safely and successfully in all tested cases, albeit sometimes with longer trajectories. Notably, our safety filter alone (without the RL policy) also ensured safety, but typically resulted in the most conservative and longest paths. }
% \vspace{-5pt}

\section{Conclusions}
In this work, we presented a novel risk-adaptive navigation approach based on CVaR-BF that leverages a dynamic zone-based barrier function with an adjustable risk level. Our method flexibly accommodates uncertainties in obstacle models, avoiding overly conservative behavior while maintaining high feasibility and low collision rates even in crowded, dynamic environments. The proposed approach achieves the highest success rate among all baselines, especially under significant uncertainty. Future work includes quantifying the estimation error of human trajectory prediction models, developing continuous-time formulations, and validating our method in more realistic navigation scenarios.

% \vspace{-5pt}

\bibliographystyle{IEEEtran}
\bibliography{refs}

% Xinyi Wang received the Ph.D. degree in Mechanical and Automation Engineering from the Chinese University of Hong Kong (CUHK), Hong Kong, in 2023 and received the B.E. degree in 2019 from Xiamen University, Xiamen, China. She is currently a Postdoctoral Researcher in the Distributed Autonomous Systems and Control (DASC) Lab at the University of Michigan (UMich). Prior to her current position, she was a Postdoctoral Researcher in the Unmanned Systems Research (USR) Group at the Hong Kong Centre for Logistics Robotics (HKCLR) and an Honorary Postdoctoral Researcher at CUHK in 2023. Her research interests include motion planning, multi-agent systems, control, and optimization.

\end{document}